\documentclass[letterpaper, 10 pt, conference]{ieeeconf}  % Comment this line out
\IEEEoverridecommandlockouts                              % This command is only
\usepackage{url}
\usepackage{amsmath,amssymb}
\usepackage{verbatim}
\usepackage{float}
\usepackage{graphicx,subfigure}
\usepackage{color}
\usepackage[colorlinks=true,linkcolor=black]{hyperref}
\usepackage{breakurl}
\usepackage{cite}
\usepackage[ruled,linesnumbered]{algorithm2e}
\usepackage{amssymb}
\usepackage{mathrsfs}

\newcommand{\minmax}{$\min\max$}

\newcommand{\changed}{}
\newcommand{\ie}{\emph{i.e.,}}
\DeclareMathOperator{\Tr}{tr}

\newtheorem{theorem}{Theorem}

\title{\LARGE\bf Non-Myopic Target Tracking Strategies for State-Dependent Noise}%
\author{Zhongshun Zhang and Pratap Tokekar%
\thanks{The authors are with the Department of Electrical \& Computer Engineering, Virginia Tech, USA. \texttt{\small \{zszhang,tokekar\}@vt.edu}.}%
\thanks{This material is based upon work supported by the National Science Foundation under Grant \#1566247.}}

\begin{document}

\maketitle
\thispagestyle{empty}
\pagestyle{empty}

%%%%%%%%%%%%%%%%%%%%%%%%%%%%%%%%%%%%%%%%%%%%%%%%%%%%%%%%%%%%%%%%%%%%%%%%%%%%%%%%
\begin{abstract}
We study the problem of devising a closed-loop strategy to control the position of a robot that is tracking a possibly moving target. The robot is capable of obtaining noisy measurements of the target's position. The key idea in active target tracking is to choose control laws that drive the robot to measurement locations that will reduce the uncertainty in the target's position.
The challenge is that measurement uncertainty often is a function of the (unknown) relative positions of the target and the robot. Consequently, a closed-loop control policy is desired which can map the current estimate of the target's position to an optimal control law for the robot. 

Our main contribution is to devise a closed-loop control policy for 
%\footnote{\changed{The version of this paper appearing in CDC '16~\cite{zhang2016nonmyopic} incorrectly claimed that Theorem 2 holds for non-linear systems. However, Theorem 2 does not hold for non-linear systems. Instead, it holds for the case of bounded variance, state-dependent noise. In this manuscript, we present the correct result. The changes between the version in \cite{zhang2016nonmyopic} and this manuscript are highlighted.}} 
target tracking that plans for a sequence of control actions, instead of acting greedily. We consider scenarios where the noise in measurement is a function of the state of the target. We seek to minimize the maximum uncertainty (trace of the posterior covariance matrix) over all possible measurements. We exploit the structural properties of a Kalman Filter to build a policy tree that is orders of magnitude smaller than naive enumeration while still preserving optimality guarantees. We show how to obtain even more computational savings by relaxing the optimality guarantees. The resulting algorithms are evaluated through simulations.
% along with proofs of correctness and bounds on optimality.

\end{abstract}

%%%%%%%%%%%%%%%%%%%%%%%%%%%%%%%%%%%%%%%%%%%%%%%%%%%%%%%%%%%%%%%%%%%%%%%%%%%%%%%%
\section{INTRODUCTION}

Tracking a moving, possibly adversarial target is a fundamental problem in robotics and has long been a subject of study~\cite{bar2004estimation,li2003survey,li2010survey,li2001survey,li2002survey,li2005survey}. Target tracking finds applications in many areas such as surveillance~\cite{rao1993fully}, telepresence~\cite{karnad2012modeling}, assisted living~\cite{montemerlo2002experiences}, and habitat monitoring~\cite{isler2015finding,tokekar2013tracking}. Target tracking refers to broadly two classes of problems: (i) estimating the position of the target using noisy sensor measurements; and (ii) actively controlling the sensor position to improve the performance of the estimator. The second problem is distinguished as \emph{active} target tracking and is the subject of study of this paper. 

The main challenge in active target tracking is that the value of future measurement locations can be a function of the unknown target state. Take as example, a simple instance of estimating the unknown position of a stationary target where the measurement noise is a function of the distance between the robot and the target. If the true location of the target were known, the robot would always choose a control sequence that drives it closer to the target. Since, in practice, the true target location is unknown, we cannot determine such a control sequence exactly. A possible strategy in this case would be to plan with respect to the probability distribution of the target. However, the probability distribution itself will evolve as a function of the actual measurement values. This becomes even more challenging if the target is mobile. We use a \changed{Kalman Filter (KF)} to estimate the state of a moving target with a possibly \changed{state-dependent} measurement model where the measurement noise is a function of the distance between the robot and the target.

When planning non-myopically (\ie{} for multiple steps in the future), one can enumerate all possible future measurements in the form of a tree. In particular, a minimax tree can be used to find the optimal (in the \emph{min-max} sense) control policy for actively tracking a target~\cite{tokekar2011active}. The size of the \minmax{} tree grows exponentially with the time horizon. The tree can be pruned using $\alpha-\beta$ pruning~\cite{russell1995modern}. Our main contribution is to show how the properties of a KF can be exploited to prune a larger number of nodes without losing optimality. In doing so, we extend the pruning techniques first proposed by Vitus et al.~\cite{vitus2012efficient} for linear systems. Using a \minmax{} tree, we generalize these results to a \changed{state-dependent, time-variant} dynamical systems. Our pruning techniques allow us to trade-off the size of the tree (equivalently, computation time) with the performance guarantees of the algorithm. We demonstrate this effect in simulations. 

The rest of the paper is organized as follows. We start with the related work in Section~\ref{sec:relwork}. The problem formulation is presented in Section~\ref{sec:probform}. Our main algorithm is presented in Section~\ref{sec:algo}. We validate the algorithm through simulations described in Section~\ref{sec:sims}. Finally, we conclude with a brief discussion of future work in Section~\ref{sec:conc}.

%%%%%%%%%%%%%%%%%%%%%%%%%%%%%%%%%%%%%%%%%%%%%%%%%%%%%%%%%%%%%%%%%%%%%%%%%%%%%%%%

\section{Related Work} \label{sec:relwork}
The target tracking problem has been studied under various settings. Bar-Shalom et al.~\cite{bar2004estimation} present many of the commonly-used estimation techniques in target tracking. The five-part survey by Li and Jilkov~\cite{li2003survey,li2010survey,li2001survey,li2002survey,li2005survey} covers commonly-used control and maneuvering techniques for active target tracking. In the rest of the section, we discuss works most closely related to our formulation.

Vitus et al.~\cite{vitus2012efficient} presented an algorithm that computes the optimal scheduling of measurements for a linear dynamical system. Their formulation does not directly model a target tracking problem. Instead, the goal is to track a linear dynamical system using a set of sensors such that one sensor can be activated at any time instance. The posterior covariance in estimating a linear system in a Kalman filter depends on the prior covariance and sensor variance but not on the actual measurement values (unlike the case in non-linear systems). Thus, one can build a search tree enumerating all possible sensor selections and choosing the one that minimizes the final covariance. The main contribution of Vitus et al. was to present a pruning technique to reduce the size of the tree while still preserving optimality.

Atanasov et al.~\cite{atanasov2014information} extended this result to active target tracking with a single robot. A major contribution was to show that robot trajectories that are nearby in space can be pruned away (under certain conditions), leading to further computational savings. This was based on a linear system assumption. In this paper, we build on these works and make progress towards generalizing the solution for \changed{state-dependent observation} systems.

%%%%%%%%%%make change
A major bottleneck for planning for \changed{state-dependent observation} systems is that the future covariance is no longer independent of the actual measurement values, as was the case in linear \changed{state-independent} systems. 
%\changed{\sout{For example, an Extended Kalman Filter (EKF) is typically used in place of Kalman Filter for estimating non-linear systems. In an EKF, the state is linearized about the current state estimate after every measurement.}}
 The covariance update equations use the \changed{linear} models and as such will depend on the state estimate and the measurements. Thus, the search tree will have to include all possible measurement values and not just all possible measurement locations. Furthermore, finding an optimal path is no longer sufficient. Instead one must find an optimal policy that prescribes the optimal set of actions for all possible measurements. We show how to use a \minmax{} tree to find such an optimal policy while at the same time leveraging the computational savings that hold for the linear case.

%%%%%%%%%%%%%%%%%%%%%%%%%%%%%%%%%%%%%%%%%%%%%%%%%%%%%%%%%%%%%%%%%%%%%%%%%%%%%%%%

\section{Problem Formulation} \label{sec:probform}
We assume that the position of the robot is known accurately using onboard sensors (\emph{e.g.}, GPS, laser, IMU, cameras). The motion model of the robot is given by:
\begin{equation}
  \label{eq1}
  X_r(t+1)=f(X_r(t),u(t))
\end{equation}
where $X_r=[x_r(t), y_r(t)]^T\in\mathbb{R}^2$ is the position of robot and $u(t)\in \mathcal{U} $ is the control input at time $t$. $\mathcal{U} $ is a finite space of control inputs. We assume there are $n$ actions available as control inputs at any time:
$$\mathcal{U}(t)=\{u_1(t),u_2(t),\cdots, u_n(t)\}.$$

The robot is mounted with a sensor that is capable of obtaining a measurement of the target. We assume that the target's motion model is given by:
\begin{equation}
  \label{eq2}
  X_o(t+1)=C_tX_o(t)+v(t)\\
\end{equation}   
where, $ X_o(t)=[    x_o(t),  y_o(t)]^T   $ is the 2-dimensional position of the target and $v(t)$ is Gaussian noise of known covariance.

The task of the robot is to track the target using its noisy measurements. The measurements, $Z(t)$, can be a nonlinear function of the states of the target and robot:
\begin{equation}
  \label{eq3}
 \changed{  Z(t)= H(X_r(t)) X_o(t)  +\omega(X_r(t), X_o(t))}
\end{equation}
The measurement noise, $\omega(t)$, is a Gaussian whose variance depends on the distance between the robot and the target:
\begin{equation}
    \label{eq6}
    \changed{ \omega(X_r(t), X_o(t))\sim N\left(0, \delta_1^2 +\delta_2^2d(X_r(t),X_o(t)) \right)}
  \end{equation}
where,\\
\leftline{$d(X_r(t),X_o(t))=$}
\begin{equation}
\left\{
\begin{aligned}
\mathcal{C},\qquad \qquad &\quad ||X_r(t)-X_o(t)||_2>\mathcal{B} \\ 
    \frac{\mathcal{C}||X_r(t)-X_o(t)||_2}{\mathcal{B}} , &\quad ||X_r(t)-X_o(t)||_2 \le \mathcal{B} \\  
\end{aligned}
\right.
\label{eqn:cov}
\end{equation}

When the true distance between the robot and target is within $\mathcal{B}$, we assume that measurement noise is proportional to the true distance. When the distance is greater than $\mathcal{B}$), the variance is assumed to be a constant maximum value of $\mathcal{C}$. 

The estimated position and the covariance matrix of the target at time $t$, $ \hat{X}_o(t)$ and  $\hat{\Sigma}_o(t)$, is given by the Kalman Filter. The uncertainty in the estimate of the target's position is measured by the trace of the covariance matrix. The problem considered in this paper can be formally stated as follows.\\
%
%$$\hat{X}_o(t)=\left[                
%  \begin{array}{c}   
%   \hat{x}_o(t) \\ \hat{y}_o(t)
%  \end{array}
%\right]    $$    

\noindent\textbf{Problem:} Given an initial robot position $X_r(0)$ and an initial target estimate $[\hat{X}_o(0), \hat{\Sigma}_o(0)]$, find a sequence of control laws for the robot, $\sigma=u_0, u_1,\cdots,  u_{T} \in \mathcal{U}^T$ from time $t=0$ to $t=T$ to minimize trace of the covariance in the target's estimate at time $t=T$. That is, 

\begin{equation}
  \label{eq5}
 \min_{u(t)} \max_{z(t)}  \quad \text{tr}(\Sigma_T) 
\end{equation}
such that,
 $$      \qquad  \Sigma_{t+1}=\rho_{{t+1}}(\Sigma_{t}),\quad t=0,1,\cdots, T-1 $$ 
where $\rho_{t}(\cdot)$ is the Kalman Riccati equation~\cite{kumar1986stochastic}. 

The Riccati equation, $\rho(\cdot)$, maps the current covariance matrix $\hat\Sigma_k$, under a new measurement to the covariance matrix at the next time step,
\begin{align}
 \rho_i(\hat{\Sigma}_k)=&C_k\hat{\Sigma}_kC_k^T-C_k\hat{\Sigma}_kH_{k}^T(H_{k}\hat{\Sigma}_kH_{k}^T+\hat{\Sigma}_{w})^{-1}H_{k}\hat{\Sigma}_kC_k^T\notag\\
 &+\Sigma_v  \label{eq6}
\end{align}
 \changed{where $H_k$ is the matrix of the measurement equation computed around ${X}_r(k)$ at time $k$.}
% $$H_k=\left.\frac{\partial  h(x_r(k), x_o(k))}{\partial   x_o(k)}\right|_{x_o(k|k-1)}$$ 
$\Sigma_{w}$ is the covariance of the measurement noise given in Equations~(\ref{eq3})--(\ref{eqn:cov}). 

The true position of the target is unknown making it impossible to determine $\Sigma_{w}$ exactly. Consequently, we use an estimate of $\Sigma_{w}$ using the estimated  target's position $\hat{X}_o(k)$. Therefore, the optimal solution for the problem defined will be a closed-loop policy that should map the estimated target's position to the optimal control action for the robot.
%

%%%%%%%%%%%%%%%%%%%%%%%%%%%%%%%%%%%%%%%%%%%%%%%%%%%%%%%%%%%%%%%%%%%%%%%%%%%%%%%%

\section{Closed-Loop Control Policy} \label{sec:algo}
References~\cite{vitus2012efficient, atanasov2014information} solve a similar problem but for a linear Gaussian system. The linearity assumption makes the Riccati equation independent of the position of the target. Consequently, they show an open loop policy can determine the optimal control sequence for the robot. In our case,  the optimal control policy for this \changed{state-dependent} observation model case will be an adaptive (closed-loop) control policy. However, this generalization comes at the expense of discretization of the set of possible target measurements. Specifically, we assume that the measurement at any time step is chosen from one of $k$ tuples of candidate measurements. That is,
$$z(t) \in \{z_1(t),z_2(t),\cdots, z_k(t)\}.$$

These candidate measurements can be obtained by, for example, sampling from the continuous distribution of zero mean sensor noise around the current estimate of the target.  For example, we can choose $k$ candidate measurements from the data within 3 standard deviation of the mean value, which contain $99.7\%$ of the possible measurements.
\begin{figure}[H] \label{c2dGuassian}
  \centering
  \includegraphics[height=4cm  ]{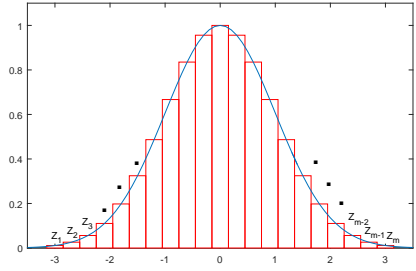}
  \caption{Obtaining a finite set of candidate measurements by discretizing the Gaussian distribution of the measurement noise.}
  \label{Minimaxtree}
\end{figure}

%Given the $k$ discrete measurements, our goal is to select the control inputs from  $t=0$ to $t=T$ to minimize the covariance matrix. 

%Since the exact measurement sequence is not known, determining (and optimizing) $\Sigma_T$ is not possible. There are two alternatives that are typically used: (i) minimize the expected trace of $\Sigma_T$; and (ii) minimize the maximum trace of $\Sigma_T$. In this paper, we focus on the second objective function. Minimizing the maximum trace can be thought of as playing a game against an adversary: The robot chooses the control actions to minimize the trace where as the adversary (i.e., nature) chooses measurement noise to maximize the trace. By optimizing the $\min\max$ trace, the robot determines the best conservative policy.

\begin{figure*}[h] \label{Minimaxtree}
  \centering
  \includegraphics[height=3.5cm  ]{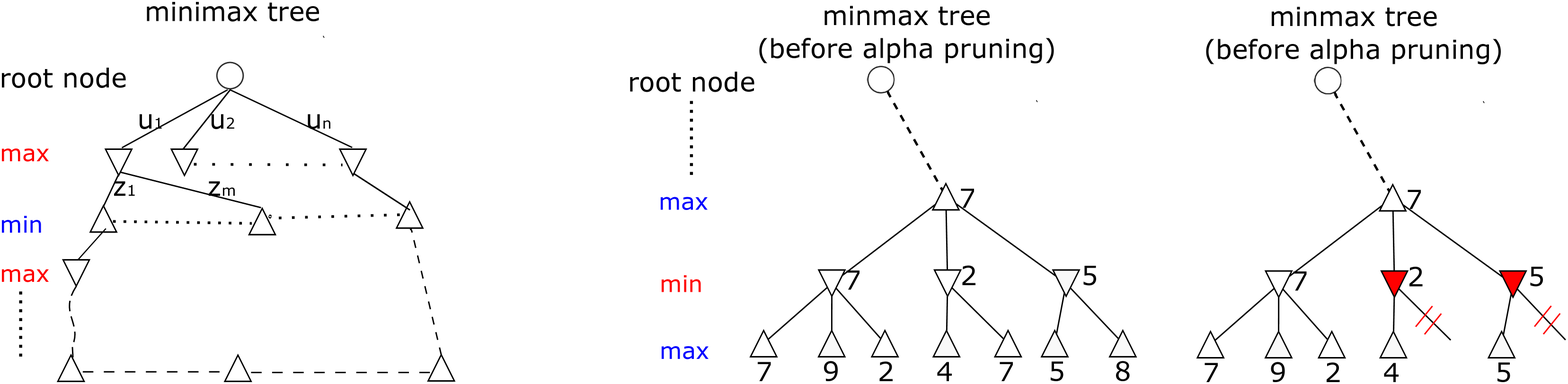}
  \caption{A minimax tree with alpha pruning. $\bigtriangledown$  and $\bigtriangleup$  are nodes in which we compute the minimum or maximum value of its children. The value at the leaf nodes equals the $\mathrm{tr}(\Sigma_k) $). $\bigtriangledown$  and $\bigtriangleup$ nodes represent control and measurement nodes, respectively. The filled $\bigtriangledown$ are pruned by alpha pruning.}
  \label{Minimaxtree}
\end{figure*}

\subsection{Optimal decisions: The minimax algorithm}
Minimizing the maximum trace can be thought of as playing a game against an adversary: The robot chooses the control actions to minimize the trace whereas the adversary (\ie{} nature) chooses measurement noise to maximize the trace. By optimizing the $\min\max$ trace, the robot determines the best conservative policy.

%Broadly speaking, an optimal minmax strategy leads to outcomes at least as good as any other strategy when the sensor noise is the worst case. 
We can find this optimal strategy by building a minimax tree. This tree enumerates all possible control laws and all possible measurements that the robot can obtain. A node on the $k$th level of the tree stores the position of the robot, $X_r(k)$, the estimated position of the target, $\hat{X}_o(k)$, and the covariance matrix $\hat{\Sigma}_k$. Each node at an odd level has one branch per control action. Each node at an even level has one branch per candidate measurement. The robot's state and the target's estimate are updated appropriately along the control and measurement branches using the state transition equation (\ref{eq1}) and the \changed{Kalman filter} update equation, respectively. The minimax value is computed at the leaf nodes and is equal to the trace of the covariance matrix at that node. These values are propagated upwards to compute the optimal strategy. Figure~\ref{Minimaxtree} shows an example. 

\begin{algorithm}
    \SetAlgoLined
    $S_0\leftarrow\left\{(X_r(o), \Sigma_0)\right\}, \quad S_t\leftarrow\phi$ for $t=1,....,T$\\
    $\mathcal{Z}=\left\{z_1,z_2,..., z_k\right\}$\\
    $\mathbf{for} \quad\ t=1:T \quad\mathbf{do}$   \\
      \quad $\mathbf{if}$ \quad NODE STATE ($\min$)\\
         \qquad $\mathbf{for \quad all} \quad\ (X_r(t-1), \Sigma(t-1))\in S_{t-1}\quad\mathbf{do}$   \\
         \qquad \qquad$\mathbf{for\quad all}\quad u_i\in \mathcal{U}$\\
         \qquad \qquad \qquad $X_r(t)\leftarrow f(X_r(t-1),u_i)$\\
        \qquad \qquad \qquad $S_t\leftarrow S_t \bigcup \left\{  (X_r(t), \Sigma(t-1))\right\}$
        
              \quad $\mathbf{else \quad if}$ \quad NODE STATE ($\max$)\\
         \qquad $\mathbf{for \quad all} \quad\ (X_r(t), \Sigma(t-1))\in S_{t}\quad\mathbf{do}$   \\
         \qquad \qquad$\mathbf{for\quad all}\quad z_i\in \mathcal{Z}$\\
         \qquad \qquad \qquad $\Sigma(t) \leftarrow \rho (X_r(t), z_i, \Sigma(t-1))$\\
        \qquad \qquad \qquad $S_t\leftarrow S_t \bigcup \left\{  (X_r(t), \Sigma(t))\right\}$\\

          $\mathbf{for} \qquad t=1:T$  do\\
        
       \quad $\mathbf{if}$ TERMINAL-TEST (Max) \\
        \qquad for each $S_t$ do\\
        \qquad  \qquad$\mathcal{V} \leftarrow (max\Tr(\Sigma_i(t)), S_t(i))$\\
        \qquad  \qquad$t \leftarrow t-1 $ \\
        
         \quad       $\mathbf{if}$ TERMINAL-TEST (Min) \\
        \qquad for each $S_t$ do\\
        \qquad  \qquad$\mathcal{V} \leftarrow (minimax\Tr(\Sigma_i(t)), S_t(i))$\\
        \qquad  \qquad$t \leftarrow t-1 $ \\
        
       $\mathbf{return} \quad\mathcal{V}$
    \caption{The minimax algorithms}
    \label{minimax} 
\end{algorithm}

\subsection{Alpha Pruning}
The number of nodes in a naive minmax tree is exponential in the depth of the tree (\ie{} in the planning horizon). As a first step in reducing the size of the tree, we implement $\alpha$ pruning~\cite{russell1995modern}. The main idea in $\alpha$ pruning is that if we have explored a part of the tree, we have an upper bound on the optimal minimax value. When exploring a new node, $n_i$, if we find that the minimax value of the subtree rooted at $n_i$ is greater than the upper bound found, that subtree does not need to be explored further. This is because an optimal strategy will never prefer a strategy that passes through $n_i$ since there exists a better control policy in another part of the tree. Note that $n_i$ must be a control node. Measurement nodes cannot be pruned since the robot has no control over the actual measurement values. The pruning algorithm is based on the general alpha-beta pruning~\cite{russell1995modern} used in adversarial games. Fig.~\ref{Minimaxtree} shows a simple example of policy tree built using alpha pruning.

\subsection{Algebraic Redundancy Pruning}    

%Our main contribution is to provide an algorithm with
%complexity lower than that of Alpha Pruning and performance better
%than that of the greedy policy. 

In addition to alpha pruning, we extend the ideas presented by~\cite{vitus2012efficient} for linear systems and extend them to get even further computational savings. If there are multiple nodes at the same level with the same $X_r(t)$ values but different target estimates, we can prune one of the nodes if it is clearly ``dominated'' by the others. The notion of domination encodes the property that that particular node will never be a part of an optimal (minmax) policy. Reference~\cite{vitus2012efficient} formalized the notion of domination by defining an algebraic redundancy constraint. We adapt this result for our notation as follows:

\begin{theorem}
[Algebraic Redundancy~\cite{vitus2012efficient}]
Let $\mathcal{H} = \{(X^j_r(t),\Sigma^j_t)\}$ be a set of $n$ nodes at the same level of the tree. If there exist non-negative constants $\alpha_1, \alpha_2, \dots, \alpha_k$ such that,
$$   \Sigma^p_t\succeq \sum^k_{i=1}{\alpha_j\Sigma^j_t}\quad \text{and} \quad \sum^k_{i=1}{\alpha_i} = 1$$
then the node $\left(X^p_r(t),\Sigma^p_t\right)$ is regarded as algebraically redundant\footnote{$M \succeq  N$ represents that $M-N$ is positive semi-definite.} with respect to $\mathcal{H}\setminus\{(X^p_r(t),\Sigma^p_t)\}$ and $(X^p_r(t),\Sigma^p_t)$ and all of its descendants can be pruned without eliminating the optimal solution from the tree.
\label{thm:fromvitus}
\end{theorem}

They prove that the trace of any successor of $(X^p_r(t),\Sigma^p(t))$ cannot be lower than one of the successors of $\mathcal{H}\setminus\{(X^p_r(t),\Sigma^p(t))\}$. Our main insight is that a similar redundancy constrained can be defined for the \changed{state-dependent} case with suitable additional constraints as described below.

\begin{theorem}[State-dependent Algebraic Redundancy] Let $\mathcal{H} = \{(X^i_r(t),\hat{X}^i_o(t),\hat{\Sigma}^i_t)\}$ be a set of $N$ nodes at the same level. If there exists a node $A=(X^A_r(t),\hat{X}^A_o(t),\hat{\Sigma}^A_t)$ at the same level such that:
\begin{enumerate}
\item the robot states are identical, \ie{}  $X^A_r(t)=X^i_r(t)$ for all $i$ in $\mathcal{H}$;
\item the least common ancestor of $A$ with any other node in $\mathcal{H}$ is a control (\ie{} min) node;
\item there exist non-negative $\alpha_i$ such that for any $K$:
\begin{equation}
 \changed{H_t\Sigma_t^AH_t^T  \succeq  \sum^{N}_{i=1}\alpha_i \left[ H_t\Sigma_t^iH_t^T + K\left(\delta_1^2 +\delta_2^2\mathcal{C}\right)\right]\\}
\end{equation}
\end{enumerate}
where, $\sum^{N}_{i=1}\alpha_i =1$, then there exists a node in $\mathcal{H}$, say $B$, such that:
\begin{equation*}
\text{tr}(\Sigma^A_{t+K}) \ge \text{tr}(\Sigma^B_{t+K}).
\end{equation*}
That is, the node $A$ can be pruned from the minimax tree without eliminating the optimal policy.
\label{thm:main}
\end{theorem}

The proof is presented in the appendix.

\subsection{Sub-optimal Pruning algorithm}  
Combining alpha pruning with \changed{linear state-independent}  algebraic redundancy pruning we can reduce a significant number of branches in the search tree while still guaranteeing optimality. We can further reduce the number of branches at the expense of losing optimality. This can be achieved by relaxing alpha-pruning and algebraic redundancy constraints. We use two parameters $\epsilon_1 > 0$ and $\epsilon_2 > 0$ as relaxation parameters for alpha pruning and algebraic redundancy pruning, respectively. In each case, we bound the loss in optimality as a function of the parameters.

Specifically, while building the tree, we prune away a node if it satisfies either of the following two conditions.  When checking for alpha pruning, we prune a node if its alpha value is greater than or equal to the best minmax value found so far minus $\epsilon_1$. Similarly, we replace the constraint in Theorem~\ref{thm:main} with the following: 

\begin{equation}
 \changed{  H_t(\Sigma_t^A+\epsilon_2)H_t^T \succeq  \sum^{N}_{i=1}\alpha_i \left[ H_t\Sigma_t^iH_t^T + K\left(\delta_1^2 +\delta_2^2\mathcal{C}\right)\right]}
\end{equation}

By varying $\epsilon_1$ and $\epsilon_2$, we can vary the number of nodes in the search tree. Next we bound the resulting loss in the optimality of the algorithm.

%%%%%%%%%%%%%%%%%%%%%%%%%%%%%%%%%%%%%%%%%%%%%%%%%%%%%%%%%%%%%%%%%%%%%%%%%%%%%%%%

\section{Error analysis}

In this section, we bound the value returned by the relaxed algorithm with respect to the optimal algorithm.

\begin{theorem}[$\epsilon_1$ Alpha Pruning]
Let $J_{2k}^\ast=\Tr(\hat{\Sigma}_{2k}^\ast)$ be the optimal minimax value returned by the full enumeration tree. If $J_{2k}^{\epsilon_1}=\Tr(\hat{\Sigma}_{2k}^{\epsilon_1})$ is the value returned by the $\epsilon_1$--alpha pruning algorithm, then
$0\leq J_{2k}^{\epsilon_1}-J_{2k}^\ast\leq\epsilon_1.$
\end{theorem} 

The proof follows directly from the fact that if a node on the optimal policy, say $n_i$ is pruned away, then the alpha value at $n_i$ is at most the alpha value of some other node, say $n_j$, that is present in the tree minus $\epsilon_1$. The alpha value of $n_j$ cannot be less than the value returned by the $\epsilon_1$ algorithm. The full proof is given in the 
appendix. 
The bound for $\epsilon_2$-algebraic redundancy pruning is more complicated.

\begin{theorem}[$\epsilon_2$ State-dependent Algebraic Redundancy] Let $J_{2k}^\ast=\Tr(\hat{\Sigma}_{2k}^\ast)$ be the optimal minimax value returned by the full enumeration tree of $2k$ levels. If $J_{2k}^{\epsilon_2}=\Tr(\hat{\Sigma}_{2k}^{\epsilon_2})$ is the value returned by the $\epsilon_2$--algebraic redundancy pruning algorithm, then
$$
0\leq  J_{2k}^{\epsilon_2}-J_{2k}^\ast \leq B^{\epsilon_2}$$
where,
\begin{align*}
&B^{\epsilon_2} =\\
&\Tr\left\{  \sum^k_{j=0} \left[ \prod^j_{i=k-1}\left( F_i(\Sigma) \Phi_{2i}(\Sigma)\right)  \prod^{k-1}_{i=j}\left(  F_i(\Sigma) \Phi_{2i}(\Sigma)\right)^T \right]\epsilon_2 
 \right\}
 \end{align*}
where, $F_i(\Sigma)=C-CK_i(\Sigma)H_i$ and $K_i(\Sigma)$ is the Kalman gain given by $K_i(\Sigma)=\Sigma H_i^T(H_i\Sigma H^T_i+\Sigma_{w})^{-1}$, and $\Phi_{2k}(\cdot)$ is the application of the Riccati equation $\rho(\cdot)$, over $k$ measurement steps: 
$$\Phi_{2k}(\cdot) = \underbrace{\rho_{2(k-1)}(\rho_{2(k-2)}( \dots \rho_0(\cdot)))}_{\text{k  steps  }\rho(\cdot)}. $$
\end{theorem}

By combining the two results, we get
$$0\leq  J_{2k}^{\epsilon_1,\epsilon_2}-J_{2k}^\ast \leq \max \left\lbrace \epsilon_1,B^{\epsilon_2} \right\rbrace.$$

%%%%%%%%%%%%%%%%%%%%%%%%%%%%%%%%%%%%%%%%%%%%%%%%%%%%%%%%%%%%%%%%%%%%%%%%%%%%%%%%

\section{Simulations} \label{sec:sims}
In this section, we present results from simulations to evaluate our pruning techniques. We carry out three types of evaluations. First, we investigate the savings of our algorithm by comparing the number of nodes in the pruned minimax tree and the full enumeration tree. Then, we study the effect of varying the $\epsilon_1$ and $\epsilon_2$ parameters on the number of nodes. Finally, we use the control policy given by our algorithm and execute it by drawing actual measurements from a random distribution. This represents a realistic scenario where the measurements are not necessarily adversarial. We demonstrate how our strategy can be used in such a case, and compare the average case performance with the worst-case performance.

%In this simulation part, the initial state and covariace matrix $[\hat{X}_o(t),  \Sigma_o(t)]$ are viewed as the starting node of the minimax tree. 

%Based on the initial position, covariance matrix and measurements. As discussed in algorithm 2, the minimax can be generated to $k$ generation and find a series of control inputs from  $t=0$ to $t=T$ to  minimize the trace of covariance matrix when the measurement is the worst case:
% $$J_T=\min\left\lbrace \max\left\lbrace  \dots\max\left( \Tr\left( \Sigma_o(T_1)\right),\dots ,\Tr\left( \Sigma_o(T_k)\right)  \right)  \right\rbrace  \right\rbrace  $$
%The odd generation generate the Min nodes, which apply the control law $\mathcal{U}$ to update the sensor or robot's state $X_o(t+1)$ and estimate target position $Z(t)$. The even generation generate the Max nodes, and the convariance matrix $\Sigma_{t+1}$ is updated based on the initial measurement and Kalman Riccati equation.

In all simulations, the robot can choose from four actions:
 $$\mathcal{U}=\left\{\left[ +e, 0\right]^T ,\left[ -e, 0\right]^T,\left[ 0, +e\right]^T, \left[0, -e\right]^T \right\}$$
 where $e$ is a constant.  $$ X_r(t+1)=X_r(t)+u(t)$$
 We build the tree using five candidate measurements at each step: $z(t)=\{z_1(t),z_2(t),\cdots, z_5(t) \}$. The five values are randomly generated with Gaussian noise.
 
\subsection{Comparing the Number of Nodes}

\begin{figure}[!htb]
  \centering
  \includegraphics[width=7cm]{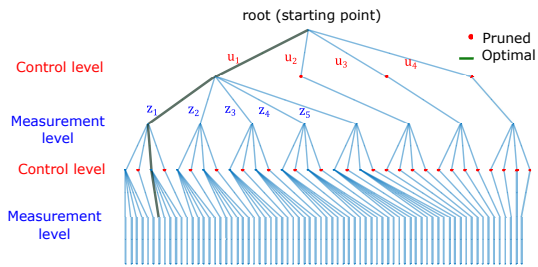}
  \caption{A five level minimax tree with pruning (189 nodes) and full enumeration (505 nodes).}
  \label{Minimax5result}
\end{figure}

\begin{figure}[!htb]
  \centering
  \includegraphics[height=5cm]{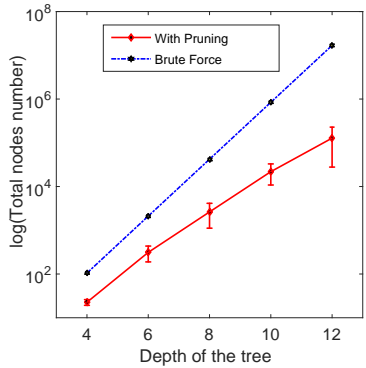}
  \caption{Comparison of the number of total nodes generated for minimax tree. Note the $\log$ scale.}
  \label{Number_of_nodes_Time_step3}
\end{figure}

Fig.~\ref{Minimax5result} and Fig.~\ref{Number_of_nodes_Time_step3} shows the number of nodes left in the minimax tree after pruning as compared to a full enumeration tree. We prune a node by comparing it to the nodes already explored. More nodes will be pruned if initial nodes encountered are ``closer'' to the optimal policy. For instance, if the first set of nodes explored happen to be control laws that drive the robot close to the target, then we expect the nodes encountered later will be pruned closer to the root. To provide a fair assessment, we generate the search trees for various true positions of the target and report the average and standard deviation of the number of nodes in Fig.~\ref{Number_of_nodes_Time_step3}.

Fig.~\ref{Number_of_nodes_Time_step3} shows that our algorithm prunes orders of magnitudes of nodes from the full enumeration tree. For a tree with depth 13, it takes $8.08\times 10^7$ to enumerate all nodes but the same optimal solution can be computed using $4.36\times 10^5$ nodes with our pruning strategy.
 \begin{figure}[!htb]  
  \centering
  \includegraphics[width=0.8\columnwidth]{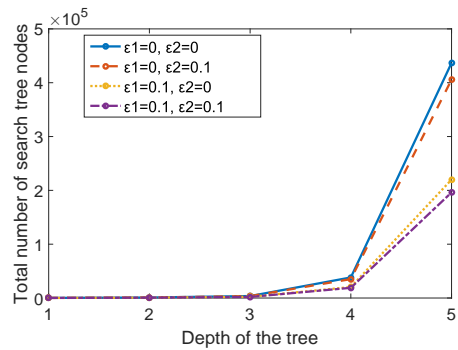}
  \caption{Effect of the $\epsilon_1$ and $\epsilon_2$ relaxation parameters on the number of nodes in the search tree. The baseline case is the optimal solution with alpha pruning and algebraic redundancy with both parameters set to zero.}
\label{Number_of_nodes_with_pruning}
\end{figure} 

Even though our algorithm prunes a large percentage of the nodes, the number of nodes still grows exponentially. By sacrificing optimality, we can prune even more nodes. We evaluate this by varying $\epsilon_1$ and $\epsilon_2$ individually first, and then jointly. As shown in Fig.~\ref{Number_of_nodes_with_pruning},  $\epsilon_1$--alpha pruning is relatively better at reducing the complexity of the minmax tree. This is intuitive because $\epsilon_1$-alpha pruning condition compares nearly every pair of nodes at the same depth. $\epsilon_2$-algebraic redundancy pruning, on the other hand, requires more conditions (same robot state) to be satisfied. Nevertheless, Fig.~\ref{Number_of_nodes_with_pruning} shows that by sacrificing optimality, the number of nodes can be substantially reduced.
  
\begin{figure*}
  \centering
  \includegraphics[height=4.5cm]{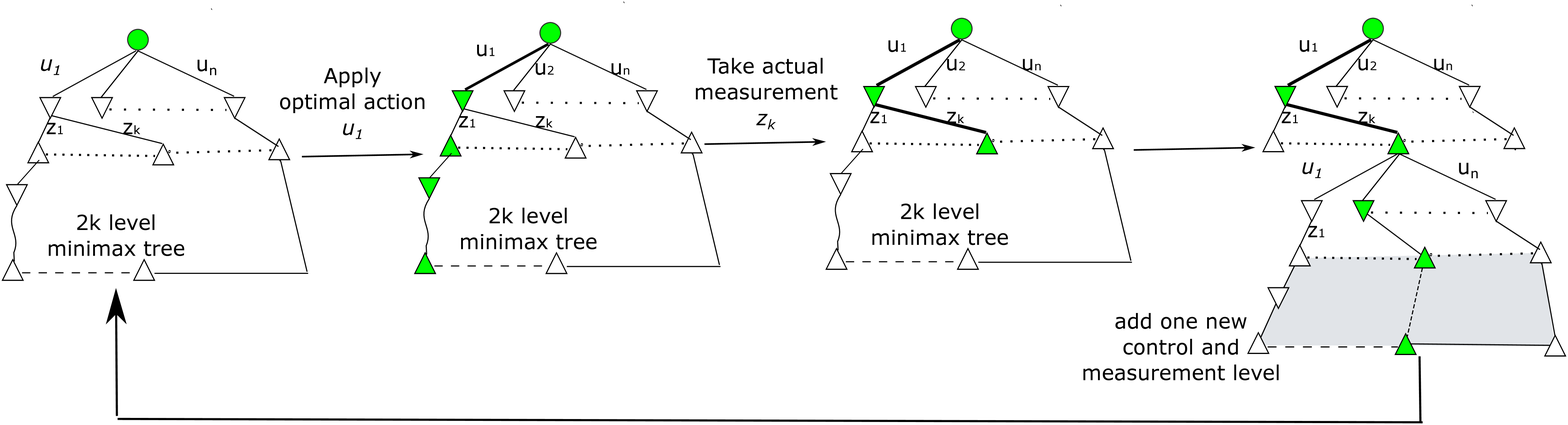}
  \caption{Online measurement update with a minmax tree. The actual measurement, $z_k$ may not correspond to a measurement node in the tree. In such a case, we choose the ``closest'' measurement in the tree.}
  \label{online}
\end{figure*}

\subsection{Online Execution of the Search Tree}

So far, we have discussed the problem of building the minimax tree. Once the tree is built, the robot can execute the optimal policy. At the root node, the robot executes the first control action along the optimal minmax path found. Next, the robot obtains a measurement. This measurement may not correspond to the worst-case measurement. Furthermore, the actual value of the measurement may not even be in the $k$ candidate measurements in the tree. The updated target estimate may not correspond to a node in the tree. Instead, we compute the distance between the actual measurement and find the closest match in the candidate set. The corresponding child node is now the new root node of the tree and the optimal policy starting at that node is executed, iteratively. Fig.~\ref{online} shows the execution in a simple instance.

%%%%%%%%%%%%%%%%%%%%%%%%%%%%%%%%%%%%%%%%%%%%%%%%%%%%%%%%%%%%%%%%%%%%%%%%%%%%%%%%

\section{Conclusion} \label{sec:conc}

We investigated the problem of devising closed-loop control policies for  \changed{state-dependent} target tracking. Unlike \changed{state-independent} tracking, the value of a candidate control law in \changed{state-dependent} target tracking is a function of the history of measurements obtained. Consequently, planning over a horizon requires taking into account possible measurement values. 
%A naive strategy is to enumerate all possible future measurements and evaluate all possible sequences of control laws. Even then, since the actual measurements are unknown, we can only minimize the expected or worst-case uncertainty in the target's estimate. 
In this paper, we focused on minimizing the worst-case uncertainty. Our solution consisted of building a $\min\max$ search tree to obtain the control policy. A full enumeration tree has size exponential in the number of measurements, control laws, and the planning horizon. Instead, we exploited the structural properties of \changed{Kalman filter} to yield a tree with significantly less number of possible nodes without sacrificing the optimality guarantees. We also showed how two parameters, $\epsilon_1$ and $\epsilon_2$, can be used to yield even more computational savings at the expense of optimality. 
%We also presented bounds on the solution sub-optimality as a function of $\epsilon_1$ and $\epsilon_2$. With this, we advance the state-of-the-art by generalizing the best-known algorithms for linear target tracking~\cite{vitus2012efficient, atanasov2014information} to non-linear scenarios.

One disadvantage of the generalization is the need to discretize the set of possible future measurements. Our immediate future work is to bound the suboptimality as a function of the number of discrete samples chosen to represent the continuous set of future measurements. Once a bound is obtained, the user may choose the correct trade-off between the computation time and the solution quality desired. A second avenue of future work focuses on extending these results to multi-robot, multi-target scenarios. Our prior work~\cite{tokekar2014multi} has shown a greedy assignment of robot trajectories to targets yield provably approximate solutions for one-step planning. We will extend this to planning over a finite horizon using the results presented in this paper.

\bibliographystyle{IEEEtran}
\bibliography{IEEEabrv,main,refs}

%%%%%%%%%%%%%%%%%%%%%%%%%%%%%%%%%%%%%%%%%%%%%%%%%%%%%%%%%%%%%%%%%%%%%%%%%%%%%%%%
\begin{appendix}
\section{Proof of Theorem \ref{Monotonicity}}
\begin{theorem}\label{Monotonicity}(Monotonicity of state dependent Riccati equation )

If two nodes of a linear stochastic system satisfy the condition:
$H\Sigma_{t}^AH^T+S^A \succeq H\Sigma_{t}^BH^T+ S^B$
and $\Sigma_{t}^A \succeq \Sigma_{t}^B$, then after apply one step Riccati map, we have:
$$\rho(\Sigma_{t}^A)\succeq \rho(\Sigma_{t}^B)$$
where, 
$$0\le S^A\le a, 0\le S^B\le a$$
\begin{align}
\rho(\Sigma_k)=&C_k\Sigma_kC_k^T-C_k\hat{\Sigma}_kH_{k}^T(H_{k}\hat{\Sigma}_kH_{k}^T+S^A)^{-1}H_{k}\hat{\Sigma}_kC_k^T\notag\\
&+\Sigma_v  \label{eq6}
\end{align}
\end{theorem}

\begin{proof} 
\begin{equation}
\begin{split}
&L(\Sigma_{k}^A)- L(\Sigma_{k}^B)\\
=&C_k\Sigma_{k}^AC_k^T-C_k\Sigma_{k}^AH^T\left(H\Sigma_{k}^AH^T+S^A\right)^{-1}H\Sigma_{k}^AC_k^T\\
&-  C_k\Sigma_{k}^BC_k^T+C_k\Sigma_{k}^BH^T\left(H\Sigma_{k}^BH^T+ S^B\right)^{-1}H\Sigma_{k}^BC_k^T  
\end{split}
\end{equation}

We define:
$$K(\Sigma):=-F\Sigma H^T(H\Sigma H^T+S)^{-1}$$
$$F(\Sigma):=F-(F\Sigma H^T)(H\Sigma H^T+S)^{-1}$$
Note that： $F(\Sigma)=F+K(\Sigma)H$, and \\
$$K(\Sigma)(F\Sigma H^T)^T=-K(\Sigma)(H\Sigma H^T+S)K(\Sigma)^T$$

Then,
\begin{equation}
\begin{split}
L&(\Sigma^A_{t})-L(\Sigma^B_{t})-F(\Sigma^A_{t})(\Sigma^A_{t}-\Sigma^B_{t})F^T(\Sigma^A_{t})\\
=&K(\Sigma^A_{t})F\Sigma^A_{t}H^T-K(\Sigma^B_{t})F\Sigma^B_{t}H^T\\
&-K(\Sigma^A_{t})H(\Sigma^A_{t}-\Sigma^B_{t})F^T-F(\Sigma^A_{t}-\Sigma^B_{t})H^TK^T(\Sigma^A_{t})\\
&-K(\Sigma^A_{t})H(\Sigma^A_{t}-\Sigma^B_{t})H^TK(\Sigma^A_{t})^T\\
=&K(\Sigma^A_{t})F\Sigma^A_{t}H^T-K(\Sigma^B_{t})F\Sigma^B_{t}H^T\\
&-K(\Sigma^A_{t})[H\Sigma^A_{t}F^T-H\Sigma^B_{t}F^T]\\
&-[F\Sigma^A_{t}H^T-F\Sigma^B_{t}H^T]K^T(\Sigma^A_{t})\\
&-K(\Sigma^A_{t})H(\Sigma^A_{t}-\Sigma^B_{t})H^TK^T(\Sigma^A_{t})\\
=&-K(\Sigma^B_{t})(F\Sigma^B_{t}H^T)^T+K(\Sigma^A_{t})(F\Sigma^B_{t}H^T)\\
&+(F\Sigma^B_{t}H^T)^TK^T(\Sigma^A_{t})\\
&+K(\Sigma^A_{t})(H\Sigma^B_{t}H^T+S^A)K^T(\Sigma^A_{t})\\
=&K(\Sigma^B_{t})(H\Sigma^B_{t}H^T+S^B)K^T(\Sigma^B_{t})\\
&+K(\Sigma^A_{t})(H\Sigma^B_{t}H^T+S^A)K^T(\Sigma^A_{t})\\
&-K(\Sigma^A_{t})(H\Sigma^B_{t}H^T+S^B)K^T(\Sigma^B_{t})\\
&-K(\Sigma^B_{t})(H\Sigma^B_{t}H^T+S^B)K^T(\Sigma^A_{t})\\
=&(K(\Sigma^B_{t})-K(\Sigma^A_{t}))(H\Sigma^B_{t}H^T+S^B)(K(\Sigma^B_{t})-K(\Sigma^A_{t}))^T\\
&+K(\Sigma^A_{t})(S^A-S^B)K^T(\Sigma^A_{t})
\end{split}
\end{equation}
\end{proof}

Define 
\begin{equation}\label{maxM}
\begin{split}
&M=\\
&\max \left((F+K(\Sigma^A_{t})H)(F+K(\Sigma^A_{t})H)^T, K(\Sigma^A_{t})K^T(\Sigma^A_{t}) \right)
\end{split}
\end{equation}
So, we have,

\begin{equation}
\begin{split}
&L(\Sigma^A_{t})-L(\Sigma^B_{t})\\
=&F(\Sigma^A_{t})(\Sigma^A_{t}-\Sigma^B_{t})F^T(\Sigma^A_{t})+K(\Sigma^A_{t})(S^A-S^B)K^T(\Sigma^A_{t})\\
&+(K(\Sigma^B_{t})-K(\Sigma^A_{t}))(H\Sigma^B_{t}H^T+S^B)(K(\Sigma^B_{t})^T-K(\Sigma^A_{t}))^T\\
=&(F+K(\Sigma^A_{t})H)(\Sigma^A_{t}-\Sigma^B_{t})(F+K(\Sigma^A_{t})H)^T\\
&+K(\Sigma^A_{t})(S^A-S^B)K^T(\Sigma^A_{t})\\
&+(K(\Sigma^B_{t})-K(\Sigma^A_{t}))(H\Sigma^B_{t}H^T+S^B)(K(\Sigma^B_{t})-K(\Sigma^A_{t}))^T\\
&\text{From [\ref{maxM}], we have}\\
\succeq & K(\Sigma^A_{t})H(\Sigma^A_{t}-\Sigma^B_{t})H^T K^T(\Sigma^A_{t})+K(\Sigma^A_{t})(S^A-S^B)K^T(\Sigma^A_{t})\\
&+(K(\Sigma^B_{t})-K(\Sigma^A_{t}))(H\Sigma^B_{t}H^T+S^B)(K(\Sigma^B_{t})-K(\Sigma^A_{t}))^T\\
= & K(\Sigma^A_{t})\left((H\Sigma^A_{t}H^T+S^A)-(H\Sigma^B_{t}H^T+S^B)\right) K^T(\Sigma^A_{t})\\
&+(K(\Sigma^B_{t})-K(\Sigma^A_{t}))(H\Sigma^B_{t}H^T+S^B)(K(\Sigma^B_{t})-K(\Sigma^A_{t}))^T\\
&\qquad \qquad \text{since $H\Sigma_{t}^AH^T+S^A \succeq H\Sigma_{t}^BH^T+ S^B$  }\\
\succeq & 0
\end{split}
\label{inequaility}
\end{equation}

\section*{ Proof of Theorem 2 }

\begin{proof} We first prove a special case when $\mathcal{H}$ consists of only one node, $B$. That is, we have:
\begin{equation*}
H_t\Sigma_t^AH_t^T \succeq H_t\Sigma_t^BH_t^T+K\cdot aI
\end{equation*}
where, $a=\delta_1^2 +\delta_2^2\mathcal{C}$.

To prove:
$$\text{tr}(\Sigma_{t+K}^A) \ge \text{tr}(\Sigma^B_{t+K}).$$

1) Show that the statement holds for $K = 1$.

When $K = 1$, $H\Sigma^A_{t}H^T \succeq H\Sigma^B_tH^T+aI$. From the Kalman Riccati map,
\begin{equation}
\begin{split}
&\Sigma^A_{t+1}=\rho_i(\Sigma^A_{t},x_r(t),\hat{x}_o(t))\\
=&C\Sigma^A_{t}C^T-C\Sigma^A_{t}H_{t}^T\cdot\\
&\left(H_{t}\Sigma^A_{t}H^T_{t}+\Sigma_{w_i}\left(x_r(t),\hat{x}^A_o(t)\right)\right)^{-1}H_{t}\Sigma^A_{t}C^T+\Sigma_v\\
& \qquad \text{Applying Theorem \ref{Monotonicity}, we have}\\
\succeq& C\Sigma^B_{t}C^T-C\Sigma^B_{t}H_{t}^T\cdot\\
&\left(H_{t}\Sigma^B_{t}H^T_{t}+\Sigma_{w_i}\left(x_r(t),\hat{x}^B_o(t)\right)\right)^{-1}H_{t}\Sigma^B_{t}C^T+\Sigma_v\\
=&\Sigma^B_{t+1}
\end{split}
\end{equation}

2) Inductive step: Show that if the claims holds for $K=M$, then it also holds for $K=M+1$. This can be done as follows:
Assume the claim holds for $K=M$. Let $\Sigma^{B'}(t)=\Sigma^B(t)+a\cdot (H^TH)^{-1}$, based on the condition of $K=M$ we have,
$$\Sigma^A_{t+M} \succeq \Sigma^{B'}_{t+M}$$
that is,
$$\Sigma^A_{t+M} \succeq \Sigma^B_{t+M}+a\cdot (H^TH)^{-1}$$
Similar to the step (1):
$$\Sigma^A_{t+M+1} \succeq \Sigma^B_{t+M+1}$$
Thereby showing that indeed $K=M+1$ holds.

Since both the base case and the inductive step have been performed, by mathematical induction, the statement holds for all natural numbers $n$.

Then, we extend the proof from comparing two nodes to arbitrary $N$ nodes case, if we have
\begin{equation}
\begin{aligned}
H_t\Sigma_t^AH_t^T  \succeq  \sum^{N}_{i=1}\alpha_i \left[ H_t\Sigma_t^iH_t^T + K\left(\delta_1^2 +\delta_2^2\mathcal{C}\right)\right]
\end{aligned}
\end{equation}

Without loss of generality, we assume $\Sigma^{B'}$ is the minimum covariance matrix (in the positive semi-definite sense) among  $\Sigma^1_t,\Sigma^2_t,...,\Sigma^N_t$.
\begin{equation}
\begin{split}
H_t\Sigma_t^AH_t^T  &\succeq  \sum^{N}_{i=1}\alpha_i \left[ H_t\Sigma_t^iH_t^T + K\left(\delta_1^2 +\delta_2^2\mathcal{C}\right)\right]\\
&\succeq  \sum^{N}_{i=1}\alpha_i \left[ H_t\Sigma_t^{B'}H_t^T + K\left(\delta_1^2 +\delta_2^2\mathcal{C}\right)\right]\\
&=H_t\Sigma_t^{B'}H_t^T + K\left(\delta_1^2 +\delta_2^2\mathcal{C}\right)
\end{split}
\end{equation}
Using the result from the induction above, after $K$ minmax tree steps, there always exist a node $B'$, such that,
\begin{equation}
\text{tr}(\Sigma_{t+K}^A) \ge \text{tr}(\Sigma^{B'}_{t+K})
\end{equation}
Therefore, node A can be pruned without reducing the optimality of the minmax tree.
\end{proof}
\section*{ Proof of Theorem 4}
\begin{proof}
For some level $i$, suppose that we prune a node on the optimal policy. We have,
$$\text{tr}(H\left(\Sigma_{2i}^{\epsilon_2}\right) H^T) \le \text{tr}(H\left(\Sigma_{2i}^{*}+\epsilon_2I\right)H^T)$$ 
From~\cite{vitus2012efficient}, we know that $\forall$ $\Sigma, Q \in \mathbb{R}^{n\times n}$ and $\epsilon \geq 0$:
$$\rho_{2i}(\Sigma+\epsilon Q)\preceq \rho_{2i}(\Sigma) + F_i(\Sigma)QF^T_i(\Sigma)\epsilon.$$

Applying to the above equation we get, 
\begin{align*}
\Phi_{2k}&(\Sigma+\epsilon_2Q) = \rho_{2(k-1)}(\Phi_{2(k-1)}(\Sigma+\epsilon_2Q))\\
=&\rho_{2(k-1)}(\rho_{2(k-2)}( \dots \rho_0(\Sigma+\epsilon_2Q)))\\
\preceq& \Phi_{2k}(\Sigma)+ \rho_{2(k-1)}(\rho_{2(k-2)}( \dots \rho_2(F_1(\Sigma)QF^T_1(\Sigma))\\
\vdots\\
=&\Phi_{2k}(\Sigma)+\\ &\left[  \prod^0_{i=k-1}\left( F_i(\Sigma) \Phi_{2i}(\Sigma)\right) Q  \prod^{k-1}_{i=0}\left(  F_i(\Sigma) \Phi_{2i}(\Sigma)\right)^T \right]\epsilon_2\\ &+o(\epsilon_2)\\
\preceq &\Phi_{2k}(\Sigma)+\\ &\left[  \prod^0_{i=k-1}\left( F_i(\Sigma) \Phi_{2i}(\Sigma)\right) Q  \prod^{k-1}_{i=0}\left(  F_i(\Sigma) \Phi_{2i}(\Sigma)\right)^T \right]\epsilon_2
\end{align*}

Let $\left\lbrace \hat{\Sigma}^\ast_i\right\rbrace^k_{i=1} $ be the series of covariance matrices along the optimal minmax trajectory. Suppose that the sequence of covariance matrices along the optimal trajectory returned by $\epsilon_2$--algebraic redundancy pruning algorithm  is  $\left\lbrace \hat{\Sigma}^{\epsilon_2}_i\right\rbrace^k_{i=1} $. We get,
$$\hat{\Sigma}^{\epsilon_2}_i \preceq \hat{\Sigma}^{\ast}_i+\epsilon_2I,\quad \forall i=1,2,\dots,k$$ 

By combining the two results, we obtain the desired bound:
\begin{align*}
&0\leq  J_{2k}^{\epsilon_2}-J_{2k}^\ast = \Tr(\hat{\Sigma}^{\epsilon_2}_k)-\Tr(\hat{\Sigma}^{\ast}_k)\\
& \leq \Tr\left\{  \sum^k_{j=0} \left[ \prod^j_{i=k-1}\left( F_i(\Sigma) \Phi_{2i}(\Sigma)\right) \prod^{k-1}_{i=j}\left(  F_i(\Sigma) \Phi_{2i}(\Sigma)\right)^T \right]\epsilon_2 
\right\} \\
&=B^{\epsilon_2}
\end{align*}
\end{proof}

\end{appendix}

\end{document}